\documentclass[nohyperref]{article}

\usepackage{microtype}
\usepackage{graphicx}
\usepackage{booktabs} 

\usepackage{hyperref}

\usepackage[accepted]{icml2022}

\usepackage{amsmath}
\usepackage{amssymb}
\usepackage{mathtools}
\usepackage{amsthm}

\usepackage[capitalize,noabbrev]{cleveref}

\theoremstyle{plain}
\newtheorem{theorem}{Theorem}[section]
\newtheorem{proposition}[theorem]{Proposition}

\newtheorem{corollary}[theorem]{Corollary}
\theoremstyle{definition}
\newtheorem{definition}[theorem]{Definition}

\theoremstyle{remark}

\usepackage[textsize=tiny]{todonotes}

\newcommand{\citef}[1]{{\small\citet{#1}}}
\usepackage{booktabs}
\usepackage{natbib}
\usepackage[american]{babel}
\usepackage[nolist]{acronym}
\DeclareUnicodeCharacter{2212}{−}
\usepackage{tikz}
\DeclareMathOperator*{\argmax}{arg\,max}
\DeclareMathOperator*{\argmin}{arg\,min}

\usepackage{pgfplots}
\usepgfplotslibrary{groupplots,dateplot}
\pgfplotsset{compat=newest,}

\let\pgfimageWithoutPath\pgfimage 
\renewcommand{\pgfimage}[2][]{\pgfimageWithoutPath[#1]{Images/TikzImages/#2}}

\newcommand{\R}{\mathbb{R}}
\newcommand{\N}{\mathbb{N}}
\newcommand{\st}{\;:\;}
\newcommand{\norm}[1]{\left\|#1\right\|}
\renewcommand{\epsilon}{\varepsilon}

\usepackage{subcaption}

\definecolor{color0}{rgb}{0.00392156862745098, 0.45098039215686275, 0.6980392156862745}
\definecolor{color1}{rgb}{0.8352941176470589, 0.3686274509803922, 0.0}
\definecolor{color2}{rgb}{0.00784313725490196, 0.6196078431372549, 0.45098039215686275}

\icmltitlerunning{Decision Region Quantification}

\begin{document}

\twocolumn[
\icmltitle{Improving Robustness against Real-World and Worst-Case Distribution Shifts through Decision Region Quantification}



\icmlsetsymbol{equal}{*}

\begin{icmlauthorlist}
\icmlauthor{Leo Schwinn\textsuperscript{*}}{fau,mila}
\icmlauthor{Leon Bungert}{bon}
\icmlauthor{An Nguyen}{fau}
\icmlauthor{Ren\'e Raab}{fau}
\icmlauthor{Falk Pulsmeyer}{fau}
\icmlauthor{Doina Precup}{mila,mcgill,deep}
\icmlauthor{Bj\"orn Eskofier}{fau}
\icmlauthor{Dario Zanca}{fau}
\end{icmlauthorlist}

\icmlaffiliation{fau}{University of Erlangen-N\"urnberg}
\icmlaffiliation{bon}{Hausdorff Center for Mathematics, University of Bonn}
\icmlaffiliation{mcgill}{McGill University}
\icmlaffiliation{mila}{Mila - Quebec AI Institute}
\icmlaffiliation{deep}{Google Deepmind}

\icmlcorrespondingauthor{Leo Schwinn}{leo.schwinn@fau.de}

\icmlkeywords{Deep Learning, Robustness, Distribution-Shift}

\vskip 0.3in
]



\printAffiliationsAndNotice{}  

\begin{abstract}
The reliability of neural networks is essential for their use in safety-critical applications. Existing approaches generally aim at improving the robustness of neural networks to either real-world distribution shifts (e.g., common corruptions and perturbations, spatial transformations, and natural adversarial examples) or worst-case distribution shifts (e.g., optimized adversarial examples). In this work, we propose the Decision Region Quantification (DRQ) algorithm to improve the robustness of any differentiable pre-trained model against both real-world and worst-case distribution shifts in the data. DRQ analyzes the robustness of local decision regions in the vicinity of a given data point to make more reliable predictions. We theoretically motivate the DRQ algorithm by showing that it effectively smooths spurious local extrema in the decision surface. Furthermore, we propose an implementation using targeted and untargeted adversarial attacks. An extensive empirical evaluation shows that DRQ increases the robustness of adversarially and non-adversarially trained models against real-world and worst-case distribution shifts on several computer vision benchmark datasets.
\end{abstract}

\section{Introduction}

Deep neural networks (DNNs) achieve remarkable results on a wide range of machine learning problems, including image classification \cite{Krizhevsky2012ImageNet}, speech processing \cite{Oord2016Wavenet}, and more. However, previous work has shown that distribution shifts of the data can severely impact the classification accuracy of such models. Distribution shifts of the data can be separated into multiple categories, such as out-of-distribution data \cite{Hendrycks17Baseline}, common corruptions \cite{hendrycks2021Faces}, and adversarial examples \cite{Madry2018Adversarial}. 

Making deep learning models robust against distribution shifts in the data is a long-standing problem \cite{quinonero2008Dataset}.
One area of research focuses on improving the robustness of models at \textit{training-time}. Here, models are trained with specific procedures (e.g., data augmentations or adversarial training), which cannot be applied to pre-trained models \cite{Madry2018Adversarial,Yin2019Fourier,Geirhos2019Texture,Hendrycks2020Augmix}. These methods usually entail vast computational overhead and require access to the training data \cite{Madry2018Adversarial,Rebuffi2021Fixing}. Another set of methods aims to improve the robustness at \textit{test-time}. However, these approaches are usually semi-supervised, rely on assumptions on the distribution of the test data, and often require additional model training at test-time \cite{Li2017Revisiting,Shorten2019Survey,Zhang2021Memo}. 
Moreover, most prior work focuses either on the robustness against real-world distribution shifts (e.g., common corruptions and perturbations, spatial transformations, and natural adversarial examples) or worst-case distribution shifts (e.g., optimized adversarial examples).


\begin{figure*}[h]
    \centering
    \begin{subfigure}[b]{0.35\textwidth}
        \centering
        \includegraphics[width=\textwidth]{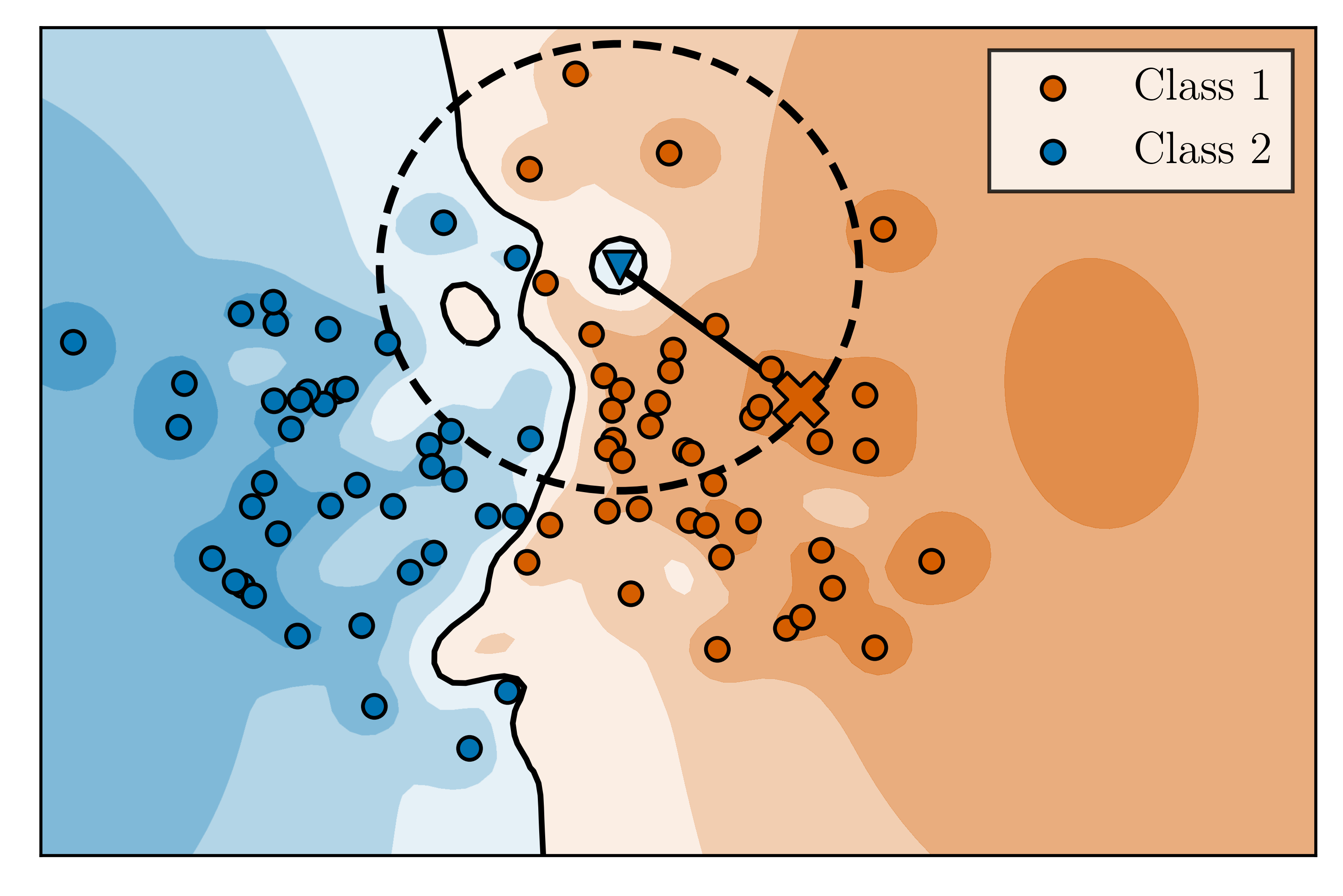}
        \caption{Calibration}
    \end{subfigure}
    \begin{subfigure}[b]{0.35\textwidth}
        \centering
        \includegraphics[width=\textwidth]{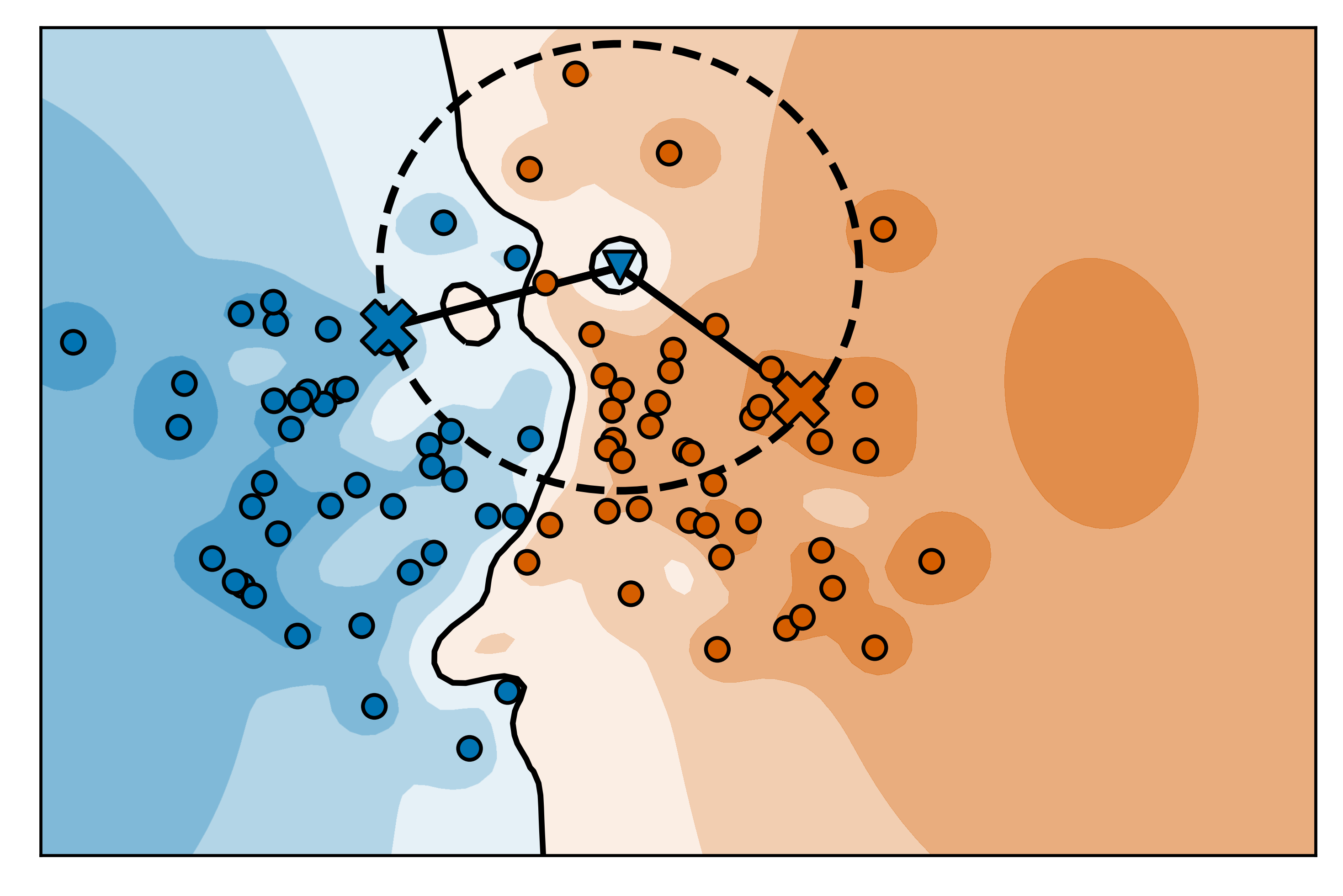}
        \caption{Exploration}
    \end{subfigure}
    \begin{subfigure}[b]{0.35\textwidth}
        \centering
        \includegraphics[width=\textwidth]{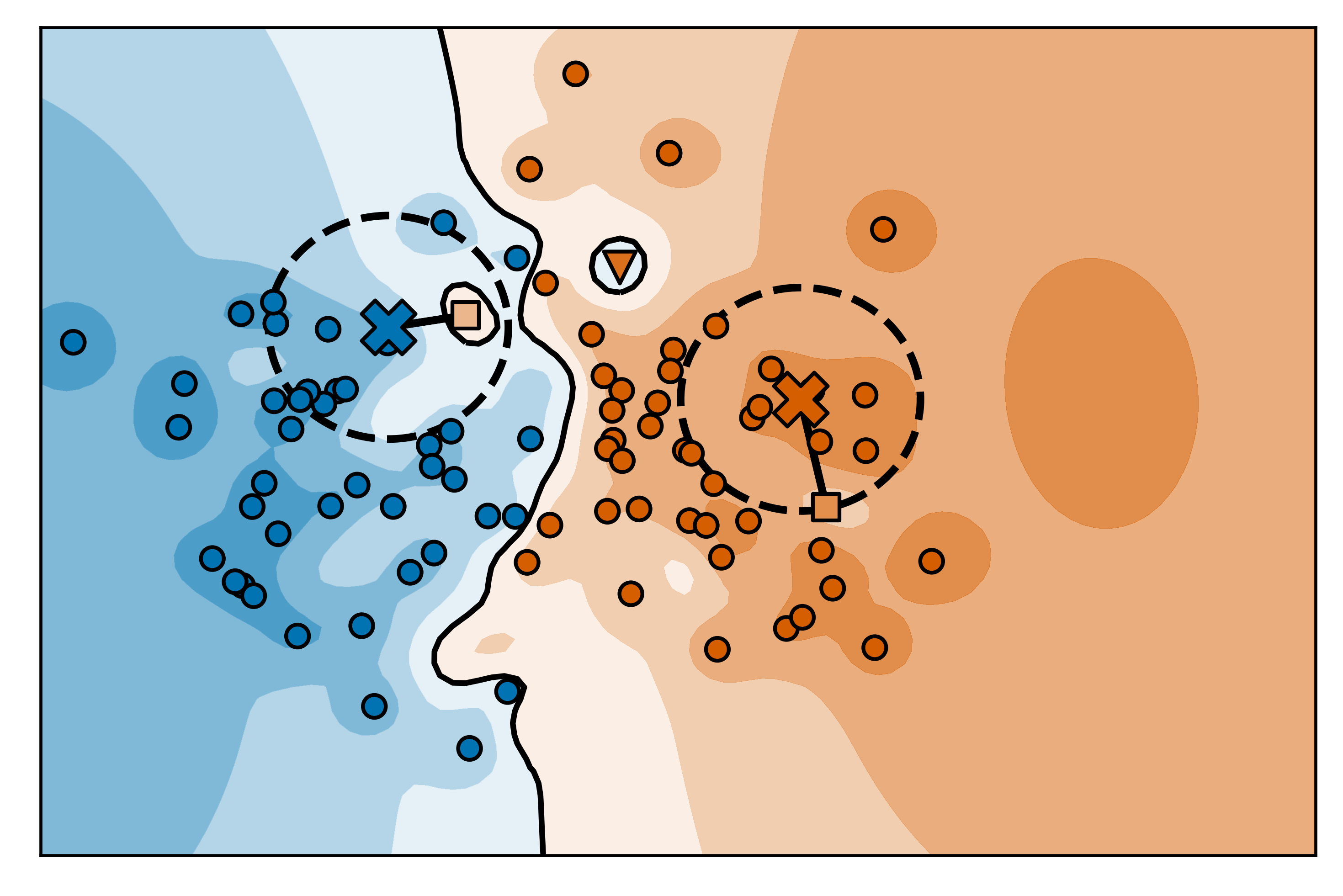}
        \caption{Quantification}
    \end{subfigure}
    \begin{subfigure}[b]{0.35\textwidth}
        \centering
        \includegraphics[width=\textwidth]{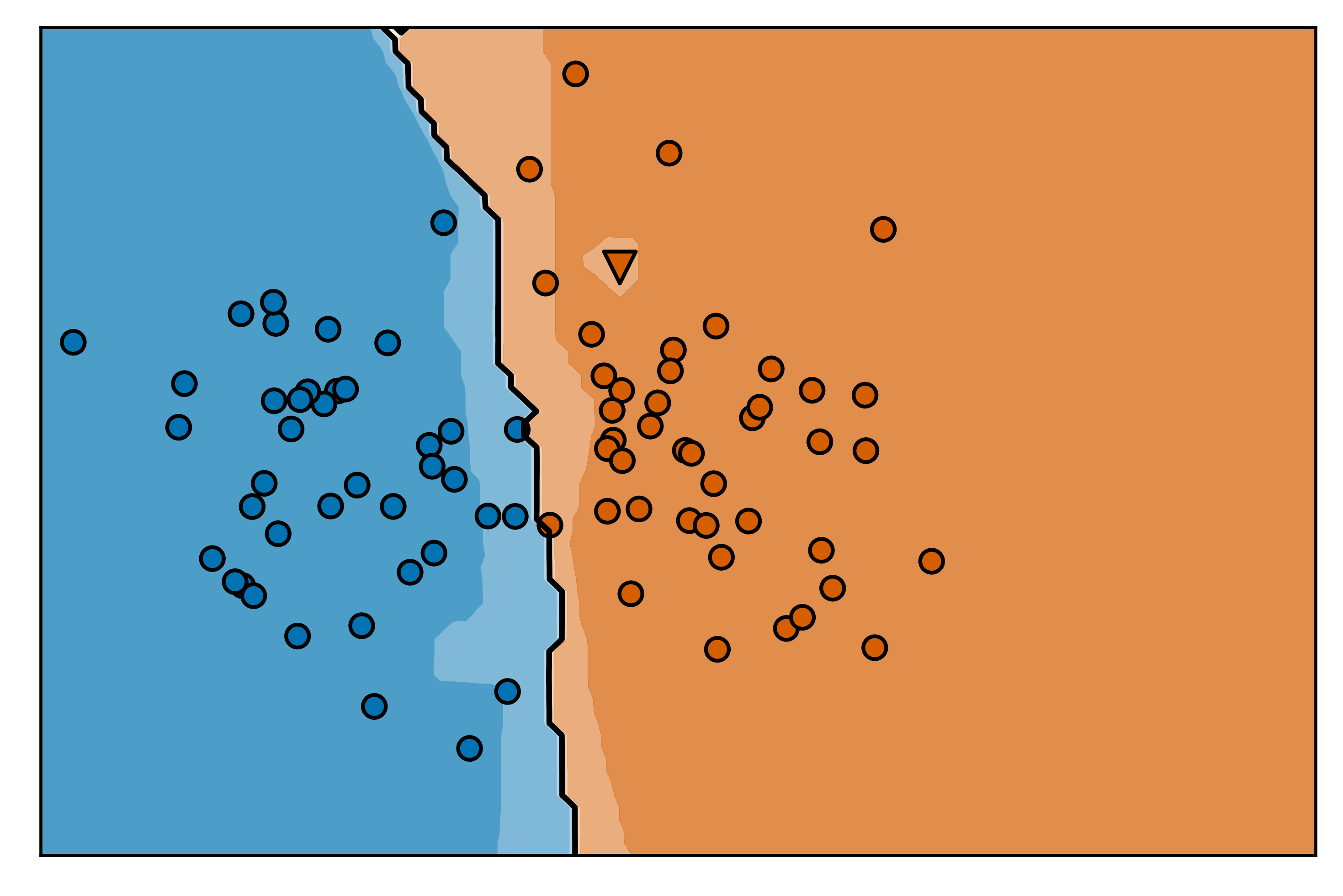}
        \caption{Decision boundary of the DRQ classifier}
    \end{subfigure}
    \caption{Illustration of the Decision Region Quantification (DRQ) algorithm. The decision boundary of an SVM classifier that is trained to separate the samples of two different two-dimensional Gaussian distributions is shown. Training samples from each class are shown as filled colored circles with a black border. One misclassified point from class $2$ is shown as a triangle instead. DRQ is applied for this point as follows: (a) In the \textit{calibration} step of the DRQ algorithm the nearest point in the feature space which is confidently classified as belonging to the other class is discovered (orange cross). The distance to this point is illustrated by the dashed circle. This distance is used as the exploration radius $\epsilon_p$ in the next step. (b) In the \textit{exploration} step, for every class, the input with the highest confidence within the exploration radius is being found (blue and orange cross). (c) In the \textit{quantification} step, for every cross discovered in the exploration step, we search for the lowest confident input within the quantification radius $\alpha \cdot \epsilon_p$ (radius indicated by the dashed circle; resulting points indicated by the colored squares). The DRQ classifier then assigns the original point to the class that exhibits the highest robustness. Here, the prediction changes from class $2$ to class $1$. (d) The resulting decision boundary after applying the DRQ algorithm to the SVM classifier for every point in the input space.}
     \label{fig:drq_algorithm}
\end{figure*}

In this work, we propose the Decision Region Quantification (DRQ) algorithm that analyses the decision surface of a given model to improve its predictions. Unlike previous work, which mainly specialized in one threat model at a time, DRQ simultaneously improves the robustness of models to both real-world and worst-case distribution shifts. The algorithm is illustrated in Figure \ref{fig:drq_algorithm}. Remarkably, the proposed approach does not require any further training data and can be directly combined with pre-trained models during test-time. Additionally, no batch statistics are used during test-time and a single test sample is sufficient for DRQ-based inference. 


Our contributions can be summarized as follows: First, we theoretically motivate the proposed DRQ algorithm by demonstrating its ability to smooth small spurious local extrema in the decision surface. Additionally, we provide an implementation using targeted and untargeted adversarial attacks. Furthermore, in an extensive empirical study, we demonstrate that DRQ consistently improves the robustness of pre-trained models against real-world distribution shifts. Moreover, we show that DRQ can be combined with other pre-processing methods to enhance the robustness even further. Lastly, we show that DRQ considerably improves robustness against worst-case distribution shifts for all adversarially-trained models in our experiments.

\section{Related Work}

The robustness of neural networks to distribution shifts is a broad research area \cite{hendrycks2021Faces}. Among the sub-fields of this research area are domain adaptation \cite{shimodaira2000improving,Wilson2020Domain}, out-of-distribution (OOD) detection \cite{Hendrycks17Baseline,Schwinn21Identifying}, corruption and perturbation robustness \cite{Hendrycks2019Benchmarking,Yin2019Fourier,Geirhos2019Texture,Zhang2021Memo}, robustness to spatial transformations \cite{Engstrom2019Spatial}, adversarial robustness \cite{Szegedy2014Intriguing,Goodfellow2015Explaining,Madry2018Adversarial,bungert2021clip}, and more. Here, we focus on robustness against common corruptions and perturbations, spatial transformations, natural adversarial examples, and optimized adversarial examples.


Most approaches that aim to improve the robustness of neural networks adapt the training phase of the model. This includes methods such as strong data augmentation \cite{Hendrycks2020Augmix} and adversarial training \cite{Goodfellow2015Explaining,Madry2018Adversarial}. Strong data augmentation has shown to be effective for improving the robustness of neural networks against common corruptions (e.g., noise, blur, digital weather corruptions) \cite{Yin2019Fourier,Hendrycks2020Augmix,hendrycks2021Faces}. However, data augmentation has proven to be largely ineffective for increasing robustness against adversarial attacks and common corruptions at the same time. Similarly, adversarial training is mainly effective against adversarial attacks and only marginally improves the robustness of neural networks against common corruptions, while it considerably reduces accuracy on clean data \cite{Raghunathan2020Tradeoff}. 

Another group of methods adjusts the inference process of models during test-time to improve their robustness. Data augmentations can be used during test-time to smooth the prediction of a model and make it more robust \cite{Krizhevsky2012ImageNet,Shorten2019Survey}. Here, the model prediction is averaged over multiple augmented versions of the input. Test-time augmentation methods only require a single test sample and no further re-training of the model. Other test-time adaptions need multiple test inputs to make distributional assumptions about the data. For instance, \citet{Li2017Revisiting} and \citet{Scheider2020Improving} calculate online batch-normalization statistics on a single batch of the test set to improve model robustness. Recently, \citet{Zhang2021Memo} proposed to increase model robustness by enforcing invariance of the model to multiple augmented versions of the same input. This is done by adapting the model parameters to minimize the entropy of a model's average output distribution for all augmentations of a given sample. Nevertheless, all these approaches are validated against real-world distribution shifts only and not against worst-case distribution shifts such as adversarial attacks.

Another related line of research are purification methods. These approaches can be used as a pre-processing step to remove data distribution shifts, such as Gaussian noise, thereby improving generalization. Denoising autoencoders (designed to remove noisy artifacts) are one of the first deep learning-based purification methods \cite{Vincent10Denoising}. Based on a similar principle, neural networks can be trained to purify adversarial examples and thus improve the robustness of another pre-trained model \cite{Samangouei18Generative,Yoon2021Score}. However, while the classifier usually does not need to be re-trained to work with the purifier model, data is still needed to train the purifier for a given task. This introduces additional computational overhead and limits applicability when no training data is available. 

To the best of our knowledge, previous test-time augmentations are designed either for real-world distribution shifts or worst-case distribution shifts. Moreover, prior approaches often require multiple test samples or additional training of the original or surrogate models.

\section{Decision Region Quantification}

In this work, we present the Decision Region Quantification (DRQ) algorithm that can be used with any pre-trained neural network. DRQ does not need additional training of the classifier or a surrogate model and works with a single test sample. Furthermore, it is designed to simultaneously work for real-world and worst-case distribution shifts. This includes common corruptions and perturbations, spatial transformations, natural adversarial examples, and optimized adversarial examples.

\subsection{Notation}

Let $N: \R^d \rightarrow \R^C$ denote a differentiable classifier, where $C\in\N$ is the number of different classes.
Let $F\colon \R^d \to \{1,\dots,C\}$, $F(x):=\argmax(N(x))$ denote the function that directly maps the input to the respective class label. We additionally define $f\colon \R^d \to \R$, $f(x):=\max(\operatorname{softmax}(N(x)))$ as the function that maps from the input to the confidence \textit{w.r.t.} the predicted class and $f_i\colon \R^d \to \R$, $f_i(x)=\operatorname{softmax}(N(x))_i$ as the function that maps from the input to the confidence  \textit{w.r.t.} the class $i \in \{1,\dots,C\}$.
Furthermore, we fix an $\ell_p$ norm $\|\cdot\|_p$ on $\R^d$ and define the ball around $x\in\R^d$ with radius $\epsilon>0$ as $B_p(x;\epsilon):=\{\Tilde{x}\in\R^d\st \|\Tilde{x}-x\|_p\leq\epsilon\}$.


\begin{definition}
    A \textit{decision region} is the subset of the input space containing all inputs $x \in \R^d$ which are classified as the same class $i$:
    \begin{align*}
        M_i:= \left\lbrace x \st  F(x)=i \right\rbrace \text{ for }i \in \lbrace 1, \ldots, C \rbrace.
    \end{align*}
\end{definition}

Every decision region can be regarded as the union of its disjunct connected components, i.e., $M_i = \cup_{j\in \mathcal{J}}  M_i^j $, where $M_i^{j_1}\ne M_i^{j_2}, \forall j_1, j_2 \in \mathcal{J}$. 

\begin{definition}
     Let $x_k \in M_i$, such that $F(x_k)=i $. The \textit{local decision region} associated with $x_k$ is the unique connected component $M_i^{j_{k}}$ which encloses the respective sample, i.e., $x_k \in M_i^{j_{k}}$.
\end{definition}

In Figure \ref{fig:drq_algorithm} the sample depicted as a triangle lies within the local decision region of class $2$, which itself lies within a larger local decision region of class $1$.


\subsection{The DRQ Algorithm}

The proposed DRQ algorithm aims to improve the robustness of a pre-trained classifier against real-world and worst-case distribution shifts. To this end, it is necessary to define robustness and how it can be quantified. An important property of a robust classification decision is that it does not change when the input is \textit{slightly} perturbed (e.g., by adding Gaussian or adversarial noise) \cite{Szegedy2014Intriguing, hendrycks2021Faces}. Thus, we define a classification decision and the associated input sample as robust if no points in direct vicinity of the sample exhibit \textit{considerably} lower confidence for the originally predicted class. 
For a given sample, DRQ quantifies the robustness of local decision region in its vicinity. Subsequently, the class label associated with the most robust local decision region is assigned. 




Specifically, the DRQ classifier $F_{\mathrm{DRQ}}:\R^d\to\{1,\dots,C\}$ takes the following steps to make a prediction for a data point $x\in\R^d$. \\

\noindent \textbf{Calibration:} In the calibration step, a vicinity $S$ of $x$ is defined.
This can be achieved by setting $S:=B_p(x;\epsilon)$ for a suitable value of $\epsilon>0$. We propose to choose $\epsilon$ by searching for the distance of the closest sample which is confidently classified as a different class, i.e., we solve
\begin{align}\label{eq:calibration}
    \epsilon_{p}:=\min_{\Tilde{x}\in\R^d}\big\lbrace \|x - \Tilde{x}\|_p\!\! \st\!\!  F(\Tilde{x})\neq F(x),\, f\left(\Tilde{x}\right) \geq c\big\rbrace,
\end{align}
where $c \in [0, 1)$ is a confidence threshold value.
The vicinity is then defined as $S:=B_p(x;\epsilon_p)$. In cases where prior knowledge about the robustness of a classifier is already available (e.g., adversarially trained models), the calibration step can be omitted and a predefined $\epsilon_p$ can be used instead (see Section \ref{sec:drq_hyperparameters}).


\noindent \textbf{Exploration:} In the exploration step, local decision regions in the vicinity of the input $x\in\R^d$ are explored.
For every class label $i \in \{1, \dots, C\}$, we search for a highest confidence sample $\Tilde{x}_i$, such that  $\Tilde{x}_i \in B_p(x;\epsilon_{p})$ and $F(\Tilde{x}_i) = i$, if it exists (otherwise we do not consider $i$ as a candidate for the correct class). For this we compute
\begin{align}\label{eq:exploration}
    \Tilde{x}_i \in \argmax_{\Tilde{x} \in B_p(x;\epsilon_{p})}\big\lbrace f(\Tilde{x}) \st  F(\Tilde{x}) = i\big\rbrace.
\end{align}




\noindent \textbf{Quantification:} In the quantification step, the robustness of all local decision regions associated with $\Tilde{x}_i$'s is evaluated to identify the most robust prediction. For every sample $\Tilde{x}_i$, we quantify the robustness of its local decision region by searching for a sample $\Hat{x} \in B_p(\Tilde{x}_i;\alpha\epsilon_p)$ for some $\alpha \in [0, 1]$ such that the confidence with respect to the $i$-th class, $f_i(\Hat{x})$, is minimized.
The final decision is then made by defining $F_\mathrm{DRQ}(x)$ as the class label that exhibits the highest robustness:
\begin{align}\label{eq:quantification}
    F_\mathrm{DRQ}(x):=\argmax_{i \in \lbrace 1, \ldots, C \rbrace} \min_{\Hat{x}\in B_p(\Tilde{x}_i;\alpha\epsilon_p)} f_i\left(\Hat{x}\right).
\end{align}

A theoretical analysis of the potential of DRQ to smooth spurious extrema is provided in Appendix \ref{app:theory} and is outlined in the following. We derive lower and upper bounds for the confidence of DRQ-based predictions after quantification depending on $\alpha$. Afterward, we show that DRQ effectively smooths spurious local extrema (i.e., misclassifications induced by distribution shifts such as noise) within the local decision region of the correct class. Moreover, we show that when a classifier is strongly convex around a local optimum, we can make precise statements about the impact of DRQ on the confidence of the classifier depending on the sharpness of the optimum. The smoothing of spurious local extrema is exemplified in Figure \ref{fig:drq_algorithm}. Here, the sample indicated by a triangle lies within a small region containing a local maximum of class $2$ within the decision region of class $1$. After applying DRQ to the classifier, the local extrema is removed and the sample is assigned to class $2$. The extent of smoothing on the decision surface by the DRQ algorithm depends directly on the choice of the vicinity $S$ (i.e., $\epsilon_p$) and the quantification radius defined by $\alpha$. Here, $\alpha \approx 0$ increases the confidence compared to the original classifier for all classes, whereas $\alpha \approx 1$ has the opposite effect and reduces the confidence. For intermediate choices of $0 < \alpha < 1$, we can expect that for spurious maxima, the confidence of the original class increases and the confidence of the spurious class decreases. Thus, in the experiments, we set $\alpha = 0.5$, as preliminary experiments confirmed this to be a good value for both real-world and worst-case distribution shifts. 


\subsection{DRQ Implementation}

In the following, we provide a baseline implementation of the DRQ algorithm for differentiable classifiers. 

\noindent \textbf{Calibration:} 
If we do not use prior information to skip the calibration step, we find the smallest possible $\epsilon_p$ in \eqref{eq:calibration} with the Fast Minimum-Norm (FMN) attack, which can be used for several different $p$-norms \cite{pintor2021fast}. Therefore, we adjust the attack to optimize for minimum-norm adversaries with predefined confidence instead of adversaries with arbitrary confidence. The norm of the final perturbation found by FMN for each sample is subsequently used as the respective exploration radius $\epsilon_p$. 


\noindent \textbf{Exploration:} The optimization problem \eqref{eq:exploration} for finding the local decision region for every class $i$ in the exploration step can be reformulated as a targeted adversarial attack:
\begin{align}\label{eq:targeted}
    \Tilde{x}_i \in\argmin_{\Tilde{x} \in B_p(x;\epsilon_p)}  \mathcal{L} (N (\Tilde{x}), y_i).
\end{align}
Here $y_i\in\R^C$ denotes the one-hot encoded vector of the $i$-th class and the loss function $\mathcal{L}$ is the categorical cross-entropy loss by which problems \eqref{eq:exploration} and \eqref{eq:targeted} become equivalent.
This process is repeated for all possible output classes $i \in {1, \dots, C}$ or can be constrained to only the most likely predictions to reduce the computational overhead. We use a Projected Gradient Descent (PGD) based attack for this step \cite{Madry2018Adversarial}. 

\noindent \textbf{Quantification:} 
Similarly, we use an untargeted PGD-based adversarial attack
\begin{align}\label{eq:untargeted}
    \hat{x}_i \in \argmax_{\Hat{x} \in B_p(\Tilde{x}_i;\alpha\epsilon_p)}\mathcal{L}(N(\Hat{x}),y_j),
\end{align}
where $y_j\in\R^C$ is the one-hot encoded vector of the class $j := F(\Tilde{x}_i)$ and $\mathcal{L}$ is the categorical cross-entropy loss, to solve the inner optimization problem in \eqref{eq:quantification} for every class $i\in\{1,\dots,C\}$. We fix the starting point of every attack to make DRQ completely deterministic and simplify the robustness evaluation with adversarial attacks. Note that it is essential to perform the exploration step for the originally predicted class $j = F(x)$ and not quantify $f_j(x)$ directly. We observed that otherwise an adversary can manipulate the distance of the given sample $x$ to the decision boundary of another class so as to change the outcome of the quantification step and mislead the DRQ-based classifier.




\section{Experiments} \label{sec:experiments}

\subsection{Data}

We evaluate DRQ for real-world and worst-case distribution shifts. These include common corruption and perturbations, natural adversarial examples, spatial transformations, image renditions, and optimized adversarial examples.

\subsubsection{Benchmark Datasets:} We use CIFAR10, CIFAR100 \cite{Krizhevsky2009}, and ImageNet \cite{deng2009imagenet} as benchmark datasets. Additionally, we include CIFAR10-C, CIFAR100-C, and ImageNet-C \cite{Hendrycks2019Benchmarking}, which are corrupted and perturbed versions of the original datasets that contain a wide variety of natural distribution shifts. For CIFAR10-C, CIFAR100-C, and ImageNet-C, we use the strongest corruption severity level of $5$. We additionally use only $1000$ images from every corruption to reduce the computational overhead. Furthermore, we include ImageNet-A, which consists of naturally occurring adversarial examples that can be classified correctly by humans but cause a considerable decrease in performance for machine learning models \cite{hendrycks2021nae}. Lastly, we include ImageNet-R~\cite{hendrycks2021Faces}, which contains various renditions of ImageNet classes (e.g., art, cartoon, or graffiti). For all ImageNet datasets, we apply DRQ only to the top-$10$ largest logits ($N(x)$) to reduce the computational overhead. Note that we do not perform any training in our experiments and only use the respective test datasets.

\subsubsection{Spatial transformations:} In addition to common benchmarks from the literature, we test DRQ against random translations and rotations of the inputs. Here, we apply random translations of up to $4$ pixels in the $x$ and $y$ coordinates for CIFAR10 and CIFAR100 (image size 32x32) and random translations of up to $20$ pixels for ImageNet (image size 224x224). For all datasets, we apply random rotations between $-20$ and $+20$ degrees. 

\subsubsection{Adversarial examples:} Lastly, we evaluate if DRQ can successfully improve the robustness of models against adversarial attacks. We use an ensemble of different attacks and report the worst-case performance against all attacks \cite{croce2020reliable}. Specifically, we consider an attack as successful if any of the ensemble attacks can induce a misclassification for the given input. Several different adversarial attacks have been proposed \cite{Goodfellow2015Explaining,Madry2018Adversarial,croce2020reliable,SchwinnExploring2021}. We choose four complementary attacks from the autoattack ensemble for our evaluation, which have shown to lead to a reliable evaluation \cite{croce2020reliable}. This includes \textit{FAB} with $100$ iterations and $9$ target classes, \textit{square} with $5000$ queries, and two versions of \textit{APGD} with either the DLR or CE loss and $100$ iterations. 

Additionally, we designed two adaptive attacks to circumvent our defense. With these attacks, we aim to find adversarial examples that show high robustness in the quantification step of DRQ and simultaneously lead to misclassification. For both attacks, we use the EOT method \cite{Athalye2018Synth} to find robust adversarial examples and use the PGD algorithm for optimization with $100$ iterations \cite{Madry2018Adversarial}. 
For both attacks, we calculate the gradient direction of the adversarial attack as the weighted average over multiple noisy gradients at every attack iteration. For the first attack, we calculate noisy gradients by sampling random noise in the $\epsilon$-ball $B_p(x;\epsilon)$ around sample $x$ and injecting it to the current adversarial example. Here, we sample $10$ noisy gradients at every attack iteration. Hence, the adversarial example found by the attack is made more robust against noise. For the second attack, we generate the noise with PGD-based adversarial attacks in the $\epsilon$-ball with $7$ iterations each. We use $7$ iterations as this has proven to be sufficient for adversarial training \cite{Madry2018Adversarial}. For the second attack, we sample only $4$ noisy gradients at every attack iteration due to the computational complexity of the attack. We refer to the first attack as PGD\textsubscript{noise} and the second attack as PGD\textsubscript{attack}.We do not early stop the adversarial attacks once the baseline model misclassifies a given input but rather construct high-confidence adversarial attacks. We additionally investigated end-to-end attacks on the whole DRQ pipeline in Appendix~\ref{app:attack_success}. As the evaluation is empirical, we encourage other researchers to create adaptive attacks that circumvent our method to update our results accordingly.

\subsection{Models}

We combine DRQ with several prior methods that aim to make models robust against real-world or worst-case distribution shifts to assess the versatility of DRQ. For our evaluations we only use pre-trained models. We use models trained to be robust against real-world distribution shifts. Here, we evaluate the model proposed by \citet{hendrycks2021Faces} that is trained with a combination of Augmix and DeepAugment (refereed to as No-AT) \cite{Hendrycks2020Augmix}. To investigate if DRQ can be combined with different adversarial training approaches, we additionally use models trained with $\ell_{\infty}$ or $\ell_{2}$-based adversarial training (referred to as $\ell_{\infty}$-AT and $\ell_{2}$-AT) \cite{Rice2020Overfitting,Wong2020Fast,Gowal2020Uncovering}. Models are either taken from the RobustBench library \cite{croce2020robustbench} or from the GitHub repositories of the authors' \cite{Hendrycks2019Benchmarking}.

\begin{table*}
\small
\caption{Accuracy [$\%$] of standard inference and DRQ inference in two different norms. The accuracy is shown for differently trained models on the CIFAR10-C, CIFAR100-C, ImageNet-C, and ImageNet-A datasets. The best accuracy is shown in \textbf{bold} and the second best is \underline{underlined}. Note that we show the mean corruption accuracy instead of the mean corruption error for the common corruption datasets (indicated with -C) to simplify the comparison with the other results.}
\centering
\begin{tabular}{lrrrr}
    \toprule
    Models & Training & Standard & DRQ $\ell_{\infty}$ & DRQ $\ell_2$  \\
    \midrule
    \textbf{CIFAR10-C} \\
    \citef{Wong2020Fast} &  $\ell_{\infty}$-AT & 66.41 & \textbf{68.44} & \underline{67.32} \\
    \citef{Gowal2020Uncovering} & $\ell_{\infty}$-AT & 72.58 & \textbf{73.12} & \underline{73.03} \\
    \citef{Rice2020Overfitting} & $\ell_{2}$-AT & \underline{71.96} & 70.08 & \textbf{73.2} \\
    \citef{hendrycks2021Faces} & No-AT & 81.33 & \underline{81.61} & \textbf{81.88} \\
    \midrule
     \textbf{CIFAR100-C} \\
    \citef{Rice2020Overfitting} & $\ell_{\infty}$-AT & 35.09 & \textbf{36.12} & \underline{35.68} \\
    \midrule 
     \textbf{ImageNet-C} \\
    \citef{Wong2020Fast} & $\ell_{\infty}$-AT & 22.82 & \textbf{23.25} & \underline{22.93} \\
    \citef{hendrycks2021Faces} & No-AT & \underline{56.86} & 56.48 & \textbf{57.48} \\
    \midrule
    \textbf{ImageNet-A} \\
    \citef{Wong2020Fast} & $\ell_{\infty}$-AT & 1.6 & \textbf{2.1} & \underline{1.9} \\
    \citef{hendrycks2021Faces} & No-AT & 3.76 & \underline{3.9} & \textbf{4.5} \\
    \midrule
    \textbf{ImageNet-R} \\
    \citef{Wong2020Fast} & $\ell_{\infty}$-AT & 31.4 & \textbf{33.3} & \underline{32.4} \\
    \citef{hendrycks2021Faces} & No-AT & \underline{46.8} & 45.2 & \textbf{47.0} \\
    \bottomrule
    \end{tabular}
    \label{tab:corruptions}
\end{table*}

\subsection{DRQ Hyperparameters} \label{sec:drq_hyperparameters}

The DRQ algorithm has five different hyperparameters: the confidence threshold $c$ for the calibration phase, the number of iterations for the calibration, search, and quantification step and the norm $\ell_p$ used to assess the robustness of a local decision region. In a preliminary experiment, we evaluated the sensitivity of the performance of DRQ to hyperparameter choices on the CIFAR10 dataset. We conducted all initial experiments on the adversarially trained model proposed in \cite{Wong2020Fast}. We did all experiments on clean data and adversarial examples generated by a standard PGD attack with $100$ iterations.

First, We explored the effect of the confidence threshold $c$ on the performance of DRQ. For adversarially trained models, we have prior knowledge about their respective robustness since they were trained to be robust against a certain adversarial budget $\epsilon$. Thus, we omitted the calibration phase for all adversarially trained models to reduce the computational overhead. For the model proposed in \cite{Wong2020Fast}, the mean $\epsilon_p$ value to misclassify a given sample was approximately equal to twice the $\epsilon$ used for adversarial training the model and we used $\epsilon_p = 2\cdot \epsilon$. This resulted in $\epsilon_{\infty} = 16/255$ and $\epsilon_{2} = 1$ for adversarially trained CIFAR models and $\epsilon_{\infty} = 8/255$ for adversarially trained ImageNet models. For the non-adversarially trained models, we set the confidence threshold to the confidence of the respective sample. Thereby, the confidence of the inputs found with DRQ in the exploration step is approximately equal to the original confidence of the model. An analysis of the number of iterations required for the different steps is given in Appendix \ref{app:iterations}. We choose the $\ell_{\infty}$ and $\ell_2$ norm to measure the robustness of the local decision regions that are commonly used to measure adversarial robustness in previous work \cite{croce2020robustbench}. An important assumption of the DRQ algorithm is that the robustness of training data is higher than the robustness of out-of-distribution data and adversarial examples. Prior work indicates that models trained to be robust against perturbations in these norms have this property \cite{RothAdversarial2020}.



\section{Results}

In this section, we summarize the results of the experiments. We combine DRQ with previously proposed training-time methods from the literature that aim to increase the robustness of neural networks. We then evaluate the robustness for real-world and worst-case distribution shifts with and without DRQ. We specifically focus on training-time methods since current test-time methods have different prerequisites and are designed for either real-world or worst-case distribution shifts but not for both.

\subsection{Common Corruptions, Natural Adversarial Examples, and Image Renditions} \label{sec:common_corruptions}

Table \ref{tab:corruptions} compares the accuracy of standard inference with the DRQ method applied in the $\ell_{\infty}$ and $\ell_2$ norm on the CIFAR10-C, CIFAR100-C, ImageNet-C, ImageNet-A, and ImageNet-R datasets. The DRQ method achieves the highest accuracy for every dataset and model. For all adversarially trained models, DRQ is most effective when the same norm is used for adversarial training and DRQ. For all non-adversarially trained models, DRQ is most effective in the $\ell_2$ norm. Both DRQ approaches are more effective than standard inference in all but two cases. For the ImageNet model proposed in \citet{Hendrycks2019Benchmarking} DRQ $\ell_{\infty}$ slightly decreases the accuracy compared to standard inference. Note that the accuracy for clean data differs only marginally between standard and DRQ-based inference in our experiments (see Appendix~\ref{app:drq_standard_network}). 

DRQ can be used with any pre-trained differentiable classifier without further requirements. Nevertheless, we explored if DRQ can be combined with prior pre-processing approaches from the literature that need additional training to further increase the robustness. Therefore, we combined DRQ with a denoising autoencoder \cite{Vincent10Denoising} trained on the CIFAR10 dataset. We specifically implemented the DnCNN architecture proposed in \cite{ZhangDenoising17}. Training details are included in Appendix \ref{app:auto}. The combination of the denoising autoencoder and DRQ increases the accuracy from $68.44\%$ to $70.74\%$ (see Table \ref{tab:denoising}). 

\begin{table}
\caption{Mean corruption accuracy [$\%$] of standard inference and DRQ in two different norms for the model proposed in \cite{Wong2020Fast} on the CIFAR10-C dataset with and without a denoising autoencoder. The best accuracy is shown in \textbf{bold} and the second best is \underline{underlined}.}
\centering
\begin{tabular}{lrrr}
    \toprule
 & Standard & DRQ $\ell_{\infty}$ & DRQ $\ell_{2}$ \\
\midrule
Model &  66.41 & \textbf{68.44} & \underline{67.32} \\
Model + Denoiser & 67.98 & \textbf{70.74} & \underline{69.92} \\
\bottomrule
    \end{tabular}
    \label{tab:denoising}
\end{table}

\begin{table*}
\small
\caption{Accuracy [$\%$] of standard inference and DRQ inference in two different norms. The accuracy is shown for differently trained models for rotations and translation of the original input. The best accuracy is shown in \textbf{bold} and the second best is \underline{underlined}. The values are arranged as follows (Standard / DRQ $\ell_{\infty}$ / DRQ $\ell_2$).}
\centering
\begin{tabular}{lrrrr}
    \toprule
    Models & Training & Rotations & Translations  \\
    \midrule
    \textbf{CIFAR10} \\
    \citef{hendrycks2021Faces} & No-AT & \underline{74.93} / 74.00 / \textbf{87.66} & \underline{92.08} / 91.65 / \textbf{92.61}  \\
    \citef{Wong2020Fast} &  $\ell_{\infty}$-AT & 43.74 / \textbf{61.31} / \underline{52.72} & 59.08 / \textbf{61.23} / \underline{61.11} \\
    \citef{Gowal2020Uncovering} & $\ell_{\infty}$-AT & 51.93 / \textbf{64.56} / \underline{58.01} & 78.16 / \textbf{78.46} / \underline{78.51} \\
    \citef{Rice2020Overfitting} & $\ell_{2}$-AT & 43.49 / \textbf{60.03} / \underline{56.20} & 63.96 / \textbf{66.26} / \underline{65.55} \\
    \midrule
    \textbf{CIFAR100} \\
    \citef{Rice2020Overfitting} & $\ell_{\infty}$-AT & 33.73 / \textbf{41.25} / \underline{39.16} & \underline{45.07} / \textbf{45.35} / 44.87 \\
    \midrule
    \textbf{ImageNet} \\
    \citef{Wong2020Fast} &  $\ell_{\infty}$-AT & 21.83 / \textbf{28.80} / \underline{23.02} & 42.51 / \underline{42.80} / \textbf{42.97} \\
    \citef{hendrycks2021Faces} & No-AT & 53.59 / \underline{54.03} / \textbf{56.73} & \underline{70.83} / 70.21 / \textbf{71.54} \\
    \bottomrule
    \end{tabular}
    ´\label{tab:spatial}
\end{table*}

\subsection{Spatial Transformations}

Table \ref{tab:spatial} compares the accuracy of standard inference with the DRQ method applied in the $\ell_{\infty}$ and $\ell_2$ norm on the CIFAR10, CIFAR100, and ImageNet datasets. DRQ improves robustness up to $17.57\%$ for rotations and up to $2.3\%$ for translations. The DRQ method achieves the highest accuracy for every dataset and model. In contrast to the evaluation on common corruptions, DRQ is not always most effective when the same norm is used for adversarial training and DRQ against spatial transformations. However, for all adversarially trained models except for the CIFAR100 model proposed by \citet{Rice2020Overfitting} both DRQ methods show higher robustness compared to standard inference. Furthermore, for adversarially trained models, using DRQ with the same norm as adversarial training always increases the robustness. As for common corruptions, DRQ is most effective in the $\ell_2$ norm for all non-adversarially trained models, 

\subsection{Adversarial Attacks} \label{sec:adversarial}

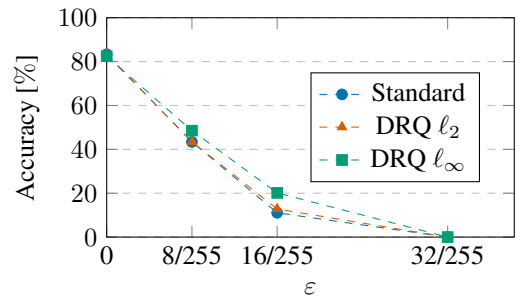
\begin{figure}[h]
\centering
\begin{tikzpicture}

\begin{axis}[
    height=4.5cm,
    width=7cm,
    xlabel={$\epsilon$},
    ylabel={Accuracy [$\%$]},
    xmin=0,
    xmax=0.15,
    xtick={0, 0.03137254901, 0.06274509803, 0.12549019607},
    xticklabels={0, 8/255, 16/255, 32/255},
    ymin=0, ymax=100,
    ytick={0, 20, 40, 60, 80, 100},
    ymajorgrids=true,
    grid style=dashed,
    legend style={at={(0.5,0.25)},anchor=south west}
]
\addplot[
    dashed,
    color=color0,
    mark=*,
    mark options={solid}
    ]
    coordinates {
    (0, 83.34)(0.03137254901, 43.37)(0.06274509803, 11.07)(0.12549019607, 0)
    };
\addlegendentry{Standard}
\addplot[
    dashed,
    color=color1,
    mark=triangle*,
    mark options={solid}
    ]
    coordinates {
    (0, 83.17)(0.03137254901, 43.39)(0.06274509803, 12.79)(0.12549019607, 0)
    };
\addlegendentry{DRQ $\ell_2$}
\addplot[
    dashed,
    color=color2,
    mark=square*,
    mark options={solid}
    ]
    coordinates {
    (0, 82.44)(0.03137254901, 48.52)(0.06274509803, 20.12)(0.12549019607, 0)
    };
\addlegendentry{DRQ $\ell_{\infty}$}
\end{axis}

\end{tikzpicture}
\caption{Worst-case accuracy [$\%$] for the model proposed by \citet{Wong2020Fast} on the CIFAR10 dataset. The accuracy is shown for standard inference, DRQ $\ell_{\infty}$, and DRQ $\ell_2$. The accuracy converges to $0$ for all methods for $\epsilon = 32/255$}
\label{fig:epsilon}
\end{figure}

\begin{table*}
\small
\caption{Worst-case accuracy [$\%$] of standard inference and DRQ inference in two different norms against an ensemble of adversarial attacks. The best accuracy is shown in \textbf{bold} and the second best is \underline{underlined}.}
\centering
\begin{tabular}{lrrrr}
    \toprule
    Models & Training & Standard & DRQ $\ell_{\infty}$ & DRQ $\ell_2$  \\
    \midrule
    \textbf{CIFAR10} \\
    \citef{Wong2020Fast} &  $\ell_{\infty}$-AT & 43.37 & \textbf{48.52} & \underline{43.39} \\
    \citef{Gowal2020Uncovering} & $\ell_{\infty}$-AT & 61.51 & \textbf{64.41} & \underline{62.94} \\
    \citef{Rice2020Overfitting} & $\ell_{2}$-AT & 66.77 & \textbf{75.16} & \underline{72.65} \\
    \citef{hendrycks2021Faces} & No-AT & 0.00 & 0.00 & 0.00 \\
    \midrule
     \textbf{CIFAR100} \\
    \citef{Rice2020Overfitting} & $\ell_{\infty}$-AT & \underline{19.01} & \textbf{24.19} & 17.79 \\
    \bottomrule
    \end{tabular}
    \label{tab:adversarial}
\end{table*}

\begin{table}
\caption{Worst-case accuracy [$\%$] of standard inference and DRQ $\ell_{\infty}$ against an ensemble of adversarial attacks. DRQ $\ell_{\infty}$ + PGD\textsubscript{noise} denotes the DRQ method, where the adversarial attack in the exploration step is replaced with PGD\textsubscript{Noise}.}
\centering
\begin{tabular}{lrrr}
    \toprule
Attacks & Clean & Worst-case \\
\midrule
Standard & 83.34 & 43.37 \\
DRQ $\ell_{\infty}$ & 82.44 & 48.52 \\
DRQ $\ell_{\infty}$ + PGD\textsubscript{noise} & 83.01 & 49.56 \\
        \bottomrule
    \end{tabular}
    \label{tab:adversarial_adaptive}
\end{table}

Table \ref{tab:adversarial} compares the accuracy of standard inference with the DRQ method applied in the $\ell_{\infty}$ and $\ell_2$ norm on the CIFAR10 and CIFAR100 datasets. Against adversarial attacks, DRQ cannot be applied to the top-n predictions only, as this could be exploited by the adversary\footnote{The adversary could for example try to make the correct class the least likely prediction}. Thus, we omitted ImageNet from this experiment due to the computational complexity of performing DRQ on all $1000$ target classes. DRQ improves the worst-case robustness up to $5.15\%$. Against adversarially trained models, DRQ is always most effective in the $\ell_{\infty}$ norm and always more effective than standard inference. For the non-adversarially trained model proposed in \cite{Hendrycks2019Benchmarking}, DRQ is not able to increase the worst-case robustness above $0\%$. While DRQ shows high robustness against some of the attacks (e.g., FAB: $84.49\%$) it is not able to defend against the adaptive attack. Out of the individual attacks, the two adaptive attacks achieved the highest success rate against DRQ-based inference in all cases (see Appendix \ref{app:attack_success} for more details). 



Prior defenses against adversarial attacks were later shown to be ineffective against stronger adaptive attacks \cite{Athalye2018Obfuscated}. Adaptive attacks against DRQ need to find robust local decision regions that belong to the wrong class in the vicinity of the input. Such attacks could be used to find more robust regions in the exploration step of the algorithm. We additionally explored if DRQ-based inference can be made more robust by using the adaptive PGD\textsubscript{noise} attack during the exploration step of the algorithm. The results are summarized in Table \ref{tab:adversarial_adaptive}. We observe that replacing standard PGD in the exploration step with the adaptive variant increases the worst-case robustness and clean accuracy of DRQ-based inference.

Furthermore, we investigated if DRQ robustness against adversarial attacks largely stems from gradient obfuscation. For this purpose, we conducted additional experiments on the CIFAR10 model proposed in \cite{Wong2020Fast} that are inspired by the findings of \cite{Athalye2018Obfuscated}. Figure \ref{fig:epsilon} demonstrates how the worst-case accuracy decreases with an increasing maximum perturbation norm $\epsilon$ for the adversarial attacks. For all inference approaches, the robustness decreases rapidly to $0\%$ against values of $\epsilon = 32/255$. We also observed an increase in the success rate of the attacks when we increased the number of attack iterations by a factor $\beta$. For $\beta=2$, the robustness of DRQ $\ell_{\infty}$ decreases from $48.52\%$ to $48.01\%$ and it further reduces to $47.93\%$ for $\beta=4$. Additionally, we conducted a random noise attack, by constructing $1000$ random perturbations within the $\epsilon$-ball, but this only reduced the accuracy marginally against DRQ. 

Lastly, we investigated the cosine similarity between DRQ direction $\gamma_{\mathrm{DRQ}} = \Hat{x}_{\mathrm{DRQ}} - x$ and the adversarial perturbation direction $\gamma = x_{\mathrm{adv}} - x$ for the model proposed by \citet{Wong2020Fast}. Here, we denote by $\Hat{x}_{\mathrm{DRQ}}$ the most robust sample found in the quantification step \eqref{eq:quantification} of DRQ. We specifically explored samples that were classified correctly by DRQ but wrongly by regular inference. The cosine similarity between the adversarial perturbation direction and the DRQ direction is negative for all attacks. This indicates that DRQ removes the perturbation generated by the adversarial attacks, thereby improving classification performance. It also supports the assumption that in-distribution data exhibits higher robustness than distribution-shifted data. More details are included in Appendix \ref{app:drq_directions}.



\section{Conclusion and Outlook}

In this paper, we propose a test-time algorithm, Decision Region Quantification (DRQ), designed to improve the robustness of any pre-trained differentiable classifier against a wide variety of distribution shifts. We theoretically motivate DRQ and show its effectiveness in an extensive empirical study which includes various models and datasets. In contrast to prior work, the proposed algorithm can handle common corruptions, spatial transformations, natural adversarial examples, image renditions, and adversarial examples simultaneously. DRQ proves to be most effective for models that already exhibit adversarial robustness in our experiments. Furthermore, DRQ still entails considerable computational overhead compared to standard inference and more effective implementations need to be explored. Additionally, while DRQ has been shown to be effective in computer vision tasks in our experiments, its performance in other application areas such as speech processing and natural language processing has yet to be evaluated. 

\section*{Author Contributions}

L.S. conceived the proposed method, conducted all experiments and wrote the initial draft of the paper. D.P., B.E., and D.Z. provided feedback to the experiments and structure of the paper. L.B. and L.S. developed the theorem and L.B. formalized it. All authors analyzed the results and contributed to the final manuscript.

\section*{Acknowledgements}

Leo Schwinn gratefully acknowledges support by a fellowship within the IFI programme of the German Academic Exchange Service (DAAD).
An Nguyen gratefully acknowledges personal funding from Siemens Healthineers. Bjoern M. Eskofier gratefully acknowledges the support of the German Research Foundation (DFG) within the framework of the Heisenberg professorship programme (grant number ES 434/8-1).

\bibliography{bib}
\bibliographystyle{icml2022}

\newpage
\appendix
\onecolumn


\section{Theoretical Analysis of DRQ for Binary Classification} \label{app:theory}

In this section we perform a theoretical analysis of DRQ in the setting of binary classification, as it was illustrated in Figure~\ref{fig:drq_algorithm}.
The general case of more than two classes is treated analogously by a one-versus-all approach.
In the binary situation the DRQ algorithm can be reformulated as follows: 
We consider a differentiable binary classifier $u:\R^d\to[0,1]$ for classifying data into the classes $0$ and $1$ with the the label decision $\operatorname{round}(u(x))$.
Here $\operatorname{round}(t)$ for $t\in[0,1]$ is the rounding operation, defined as $\operatorname{round}(t)=0$ if $t<0.5$ and $\operatorname{round}(t)=1$ if $t\geq 0.5$.
We, furthermore, fix a calibration value of $\epsilon>0$ and a parameter $\alpha\in[0,1]$.
We define the exploration and quantification steps associated with an input $x\in\R^d$ as
\begin{alignat*}{2}
    &\Tilde{x}_0 \in \argmin_{\Tilde{x}\in B_p(x;\epsilon)} u(\Tilde{x}), \qquad
    &&\Tilde{x}_1 \in \argmax_{\Tilde{x}\in B_p(x;\epsilon)} u(\Tilde{x}), \\
    &\Hat{x}_0^\alpha \in \argmax_{\Hat{x}\in B_p(\Tilde{x}_0;\alpha\epsilon)} u(\Hat{x}), \qquad
    &&\Hat{x}_1^\alpha \in \argmin_{\Hat{x}\in B_p(\Tilde{x}_1;\alpha\epsilon)} u(\Hat{x}).
\end{alignat*}
The label decision is now taken based on which of the two points $\hat{x}_0^\alpha$ and $\hat{x}_1^\alpha$ realizes the smaller value in 
\begin{align*}
    \min\big\lbrace u(\hat{x}_0^\alpha),1-u(\hat{x}_1^\alpha)\big\rbrace.
\end{align*}
In other words, the DRQ classifier with parameter $\alpha\in[0,1]$ acts as follows:
\begin{align}\label{eq:binary_DRQ}
    u_\mathrm{DRQ}^\alpha(x) = 
    \begin{cases}
    u(\hat{x}_0^\alpha),\quad \text{if }u(\hat{x}_0^\alpha)\leq 1-u(\hat{x}_1^\alpha),\\
    u(\hat{x}_1^\alpha),\quad \text{if }u(\hat{x}_0^\alpha)> 1-u(\hat{x}_1^\alpha),
    \end{cases}
    \qquad
    F_\mathrm{DRQ}^\alpha(x) = \operatorname{round}(u_\mathrm{DRQ}^\alpha(x)).
\end{align}
We can prove the following statement regarding the behavior of DRQ with respect to the parameter $\alpha\in[0,1]$.
\begin{proposition}\label{prop:monotonicity}
If $0 \leq \underline{\alpha} \leq {\overline{\alpha}} \leq 1$ then for all $x\in\R^d$ it holds
\begin{align}\label{ineq:alpha_monotone}
    u(\hat{x}_0^{\underline{\alpha}}) \leq u(\hat{x}_0^{\overline{\alpha}})
    \quad
    \text{and}
    \quad
    u(\hat{x}_1^{\underline{\alpha}}) \geq u(\hat{x}_1^{\overline{\alpha}}).
\end{align}
Furthermore, we have the following two estimates for the special cases ${\underline{\alpha}}=0$ and ${\overline{\alpha}}=1$:
\begin{align}\label{ineq:extreme_cases}
    u(\hat{x}_0^0) \leq u(x) \leq 
    u(\hat{x}_0^1)
    \quad
    \text{and}
    \quad
    u(\hat{x}_1^0) \geq u(x) \geq 
    u(\hat{x}_1^1).
\end{align}
\end{proposition}
\begin{proof}
Since $\underline{\alpha}\leq\overline{\alpha}$ it holds
\begin{align*}
    u(\hat{x}_0^{\underline{\alpha}}) = \max_{B_p(\Tilde{x}_0;\underline{\alpha}\epsilon)} u 
    \leq
    \max_{B_p(\Tilde{x}_0;\overline{\alpha}\epsilon)} u
    =
    u(\hat{x}_0^{\overline{\alpha}}).
\end{align*}
Analogously, it holds
\begin{align*}
    u(\hat{x}_1^{\underline{\alpha}}) = \min_{B_p(\Tilde{x}_1;\underline{\alpha}\epsilon)} u 
    \geq
    \min_{B_p(\Tilde{x}_1;\overline{\alpha}\epsilon)} u
    =
    u(\hat{x}_1^{\overline{\alpha}}).
\end{align*}
This proves the inequalities \eqref{ineq:alpha_monotone}.
Regarding the special cases one notes that, since $B_p(\Tilde{x}_i,0\epsilon) = \{\Tilde{x}_i\}$, it obviously holds $\hat{x}_i^0=\Tilde{x}_i$ for $i\in\{0,1\}$.
Hence, it holds
\begin{align*}
    u(\Hat{x}_0) = u(\Tilde{x}_0) = \min_{B_p(x;\epsilon)} u \leq u(x)
    \quad
    \text{and}
    \quad
    u(\hat{x}_1) = u(\Tilde{x}_1) = \max_{B_p(x;\epsilon)} u \geq u(x),
\end{align*}
and analogously
\begin{align*}
    u(\hat{x}_0^1) = \max_{B_p(\Tilde{x}_0,\epsilon)} u \geq u(x)
    \quad
    \text{and}
    \quad
    u(\hat{x}_1^1) = \min_{B_p(\Tilde{x}_1,\epsilon)} u \leq u(x),
\end{align*}
where we utilized that, since $\Tilde{x}_i\in B_p(x;\epsilon)$, it also holds $x \in B_p(\Tilde{x}_i,\epsilon)$ for $i\in\{0,1\}$.
This proves \eqref{ineq:extreme_cases}.
\end{proof}

Inequality \eqref{ineq:alpha_monotone} in Proposition~\ref{prop:monotonicity} shows that with increasing values of $\alpha$, the confidences of the DRQ algorithm decrease for both classes.
Furthermore, for $\alpha=0$ DRQ increases confidences compared to the original binary classifier $u$, whereas for $\alpha=1$ confidences are decreased.
To prevent DRQ from being over- or underconfident, it therefore seems natural to choose $\alpha=0.5$, which is also confirmed by our empirical results.

We now theoretically investigate the potential of DRQ to smooth out spurious misclassifications and start with the following definition.
\begin{definition}
Let $u:\R^d\to\R$ be a function and $\epsilon>0$.
A point $x\in\R^d$ is called $\epsilon$-wide local minimum if
\begin{align*}
    u(x) \leq u(\Tilde{x})\quad\forall\Tilde{x}\in B_p(x;\epsilon)\setminus\{x\}.
\end{align*}
Analogously, $x$ is called $\epsilon$-wide local maximum if
\begin{align*}
    u(x) \geq u(\Tilde{x})\quad\forall\Tilde{x}\in B_p(x;\epsilon)\setminus\{x\}.
\end{align*}
An $\epsilon$-wide local minimum or maximum is called \emph{strict} if these inequalities are strict.
\end{definition}
Note that every $\overline\epsilon$-wide local extremum is also an $\underline{\epsilon}$-wide local extremum for $0<\underline{\epsilon}<\overline{\epsilon}$.
Furthermore, $u$ can have several other local extrema in the ball $B_p(x;\epsilon)$ which does not change the fact that $x$ is an $\epsilon$-wide local extremum.
For simplicity we restrict the following discussion to local minima of a classifier $u:\R^d\to[0,1]$.

The first observation is a trivial one:
\begin{proposition}
Let $x\in\R^d$ be a $2\epsilon$-wide local minimum of the classifier $u:\R^d\to[0,1]$ and let $\alpha\in[0,1]$.
Then it holds 
\begin{align}
    u(\hat{x}_0^\alpha) \geq u(x)
    \quad
    \text{and}
    \quad
    u(\hat{x}_1^\alpha) \geq u(x)
\end{align}
and therefore
\begin{align}
    u_\mathrm{DRQ}^\alpha(x) \geq u(x).
\end{align}
\end{proposition}
\begin{proof}
Since $\Tilde{x}_i\in B_p(x;\epsilon)$ and therefore $\hat{x}_i^\alpha\in B_p(\Tilde{x}_i;\alpha\epsilon)\subset B_p(x;2\epsilon)$ it trivially holds
\begin{align*}
    u(\hat{x}_0^\alpha) = \max_{B_p(\Tilde{x}_0;\alpha\epsilon)} u \geq u(x), \qquad
    u(\hat{x}_1^\alpha) = \min_{B_p(\Tilde{x}_1;\alpha\epsilon)} u \geq u(x),
\end{align*}
which implies $u_\mathrm{DRQ}^\alpha(x)\geq u(x)$.
\end{proof}
For a more precise analysis, we have to quantify the sharpness of a local minimum.
If $x\in\R^d$ is a $\epsilon$-wide local minimum of $u$, we measure its sharpness by the non-negative and non-decreasing function
\begin{align}\label{eq:sharpness}
    g_\epsilon : [0,\infty) \to [0,1],\quad
    g_\epsilon(\alpha) := \max_{B_p(x;\alpha\epsilon)} u - \min_{B_p(x;\epsilon)}u.
\end{align}
Note that $g_\epsilon(0)=0$ and $g_\epsilon(1)$ describes the oscillation of $u$ on $B_p(x;\epsilon)$.
We now prove that, if the classifier $u$ is strongly convex around a local minimum, then DRQ pushes the classification of $x$ from class $0$ to class $1$ based on the sharpness of the local minimum, measured through \eqref{eq:sharpness} and the constant of strict convexity.
\begin{proposition}\label{prop:bounds}
Let $x\in\R^d$ be a $2\epsilon$-wide local minimum of the classifier $u:\R^d\to[0,1]$ and let $\alpha\in[0,1]$.
Assume that $u$ is $\mu$-strongly convex on $B_p(x;2\epsilon)$, meaning that there exists a constant $\mu>0$ such that
\begin{align}
    u(\hat{x}) \geq u(\Tilde{x}) + \langle \nabla u(\Tilde{x}),\hat{x}-\Tilde{x}\rangle + \mu\norm{\hat{x}-\Tilde{x}}_p^2,\quad\forall\hat{x},\Tilde{x}\in B_p(x;2\epsilon).
\end{align}
Then it holds
\begin{align}\label{ineq:estimate_DRQ_str_conv}
    u_\mathrm{DRQ}^\alpha(x) \geq
    \min
    \left\lbrace
    u(x) + g_\epsilon(\alpha),
    u(x) + \mu(1-\alpha)^2\epsilon^2.
    \right\rbrace
\end{align}
\end{proposition}
\begin{proof}
Since $u$ is $\mu$-strongly convex, $x$ is in fact a strict $\epsilon$-wide local minimum, meaning $\argmin_{B_p(x;\epsilon)}u  =\{x\}$.
Hence, we get $\Tilde{x}_0=x$ and by definition of $g_\epsilon$ it therefore holds 
\begin{align}\label{ineq:bound1}
    u(\hat{x}_0^\alpha) = \max_{B_p(\Tilde{x}_0;\alpha\epsilon)} u = \max_{B_p(x;\alpha\epsilon)} u = u(x) + g_\epsilon(\alpha).
\end{align}
Using the strong convexity and the fact that $\nabla u(x)=0$---since $x$ is a local minimum---it also holds
\begin{align*}
    u(\hat{x}_1^\alpha) \geq u(x) + \langle\underbrace{\nabla u(x)}_{=0},\hat{x}_1^\alpha-x\rangle+\mu\norm{\hat{x}_1^\alpha-x}_p^2 = u(x) + \mu\norm{\hat{x}_1^\alpha-x}_p^2.
\end{align*}
Now by definition it holds $\Tilde{x}_1\in\argmax_{B_p(x;\epsilon)}u$ and $\hat{x}_1^\alpha\in\argmin_{B_p(\Tilde{x}_1,\alpha\epsilon)}u$.
The strong convexity of $u$ implies that $\Tilde{x}_1$ lies on the boundary of $B_p(x;\epsilon)$. 
Hence, $\hat{x}_1^\alpha$ has distance at least $(1-\alpha)\epsilon$ from $x$ which can be seen as follows:
\begin{align*}
    \norm{\hat{x}_1^\alpha - x}_p \geq \norm{\Tilde{x}_1-x}_p-\norm{\hat{x}_1^\alpha - \Tilde{x}_1}_p\geq
    \epsilon - \alpha\epsilon = (1-\alpha)\epsilon.
\end{align*}
Using this in the estimate above we find
\begin{align}\label{ineq:bound2}
    u(\hat{x}_1^\alpha) \geq u(x) + \mu(1-\alpha)^2\epsilon^2
\end{align}
and, combing \eqref{ineq:bound1} and \eqref{ineq:bound2}, the DRQ classifier \eqref{eq:binary_DRQ} satisfies \eqref{ineq:estimate_DRQ_str_conv} as claimed.
\end{proof}
Again, the bound in \eqref{ineq:estimate_DRQ_str_conv} shows that $\alpha=0.5$ is a balanced choice:
For $\alpha=0$ the confidence of class $0$ remains $u(x)$ where as the one for class $1$ increases by $\mu\epsilon^2$.
In contrast for $\alpha=1$ the confidence of class $0$ decreases whereas the one for class $1$ satisfies the bound $u(\hat{x}_1^1)\geq u(x)$.
Together with the converse inequality from Proposition~\ref{prop:monotonicity} this implies $u(\hat{x}_1^1)=u(x)$, meaning that the confidence of class $1$ remains constant.
For in-between choices of $\alpha$ one can therefore expect that for spurious minima the confidence of class $0$ decreases and the one for class $1$ increases, which is the desired effect of DRQ.

The choice $\alpha=0.5$ is further motivated by the following corollary.
\begin{corollary}
Under the conditions of Proposition~\ref{prop:bounds} it holds
\begin{align}\label{ineq:estimate_DRQ_str_conv_weaker}
    u_\mathrm{DRQ}^\alpha(x) \geq
    \min
    \left\lbrace
    u(x) + \mu\alpha^2\epsilon^2,
    u(x) + \mu(1-\alpha)^2\epsilon^2,
    \right\rbrace
\end{align}
which for $\alpha=0.5$ yields
\begin{align}
    u_\mathrm{DRQ}^\alpha(x) \geq
    u(x) + \frac{\mu}{4}\epsilon^2.
\end{align}
\end{corollary}
\begin{proof}
Using the strong convexity one can estimate the sharpness $g_\epsilon(\alpha)$ from below.
Let $\hat{x}\in \argmax_{B_p(x;\alpha\epsilon)}u$. 
Then it holds
\begin{align*}
    g_\epsilon(\alpha) = u(\hat{x}) - u(x) \geq \mu\norm{\hat{x}-x}_p^2.
\end{align*}
Since $u$ is strongly convex, $\hat{x}$ lies on the boundary of $B_p(x;\alpha\epsilon)$ and hence
\begin{align*}
    g_\epsilon(\alpha) \geq \mu\alpha^2\epsilon^2.
\end{align*}
This then yields the estimates.
\end{proof}

\section{DRQ Hyperparameters} \label{app:hyper}

\subsection{Exploration and Quantification Iterations} \label{app:iterations}

Figure \ref{fig:iterations} shows the accuracy of DRQ for clean and adversarial data. Subfigure (a) shows the accuracy for a fixed number of quantification iterations and increasing exploration iterations. In contrast, the exploration iterations are fixed in subfigure (b) while the quantification iterations are increased.

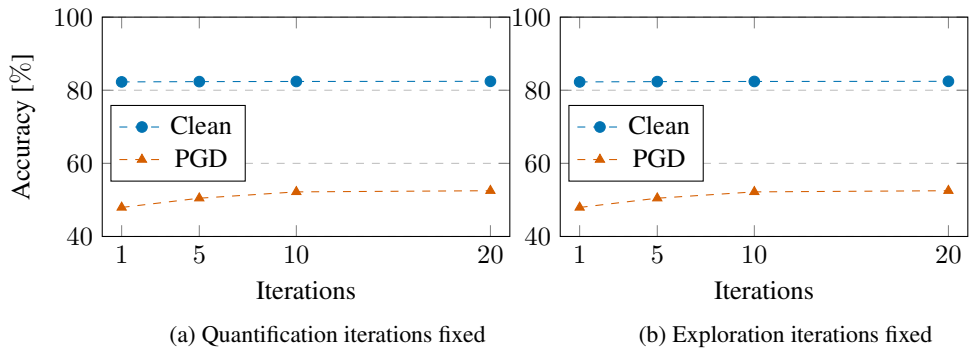
\begin{figure*}[h] 
\centering
\begin{subfigure}{.35\textwidth}
    \begin{tikzpicture} [
trim axis left, trim axis right
]

\begin{axis}[
    height=4.5cm,
    width=7cm,
    xlabel={Iterations},
    ylabel={Accuracy [$\%$]},
    xmin=0,
    xmax=21,
    xtick={1, 5, 10, 20},
    ymin=40, ymax=100,
    ytick={40, 60, 80, 100},
    ymajorgrids=true,
    grid style=dashed,
    legend style={at={(0.02,0.25)},anchor=south west}
]
\addplot[
    dashed,
    color=color0,
    mark=*,
    mark options={solid}
    ]
    coordinates {
    (1, 82.28)(5, 82.36)(10,82.40)(20,82.44)
    };
\addlegendentry{Clean}
\addplot[
    dashed,
    color=color1,
    mark=triangle*,
    mark options={solid}
    ]
    coordinates {
    (1, 47.94)(5, 50.47)(10,52.20)(20,52.52)
    };
\addlegendentry{PGD}
\end{axis}

\end{tikzpicture}
    \caption{Quantification iterations fixed}
\end{subfigure}
\begin{subfigure}{.35\textwidth}
    \begin{tikzpicture} [
trim axis left, trim axis right
]

\begin{axis}[
    height=4.5cm,
    width=7cm,
    xlabel={Iterations},
    xmin=0,
    xmax=21,
    xtick={1, 5, 10, 20},
    ymin=40, ymax=100,
    ytick={40, 60, 80, 100},
    ymajorgrids=true,
    grid style=dashed,
    legend style={at={(0.02,0.25)},anchor=south west}
]
\addplot[
    dashed,
    color=color0,
    mark=*,
    mark options={solid}
    ]
    coordinates {
    (1, 82.28)(5, 82.36)(10,82.40)(20,82.44)
    };
\addlegendentry{Clean}
\addplot[
    dashed,
    color=color1,
    mark=triangle*,
    mark options={solid}
    ]
    coordinates {
    (1, 47.94)(5, 50.47)(10,52.20)(20,52.52)
    };
\addlegendentry{PGD}
\end{axis}

\end{tikzpicture}
    \caption{Exploration iterations fixed}
\end{subfigure}
\caption{Accuracy for clean images and PGD attacked images for the model proposed by \citet{Wong2020Fast} on the CIFAR10 dataset with the DRQ $\ell_{\infty}$ method. The graph shows the accuracy for: (a) A fixed number of $20$ quantification step iterations and variable exploration iterations and (b) A fixed number of $20$ exploration step iterations and variable quantification iterations. The performance increases monotonically with the amount of iterations till it converges at about $20$ iterations for both steps.}
\label{fig:iterations}
\end{figure*}

\subsection{Confidence Threshold} \label{app:conf}

Figure \ref{fig:confidence} illustrates how the calibration distance $\epsilon_p$ is dependent on the confidence threshold $c$. The higher the confidence threshold, the harder it is for the FMN attack to find an adversarial example, and hence the larger the mean calibration distance becomes.  

\begin{figure}[h]
  \centering
  \begin{subfigure}{.3\textwidth}
      \centering
      \input{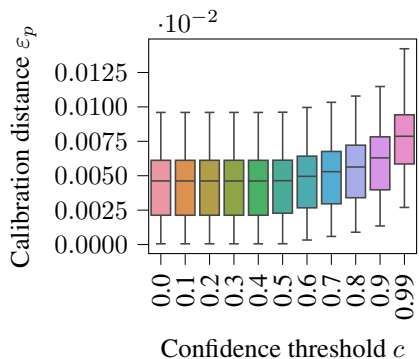}
  \end{subfigure}
  \caption{Calibration distance $\epsilon_p$ for different confidence thresholds $c$ for the non-adversarially trained CIFAR10 model proposed in \cite{Hendrycks2019Benchmarking}. The calibration distance $\epsilon_p$ remains largely constant for values $c \leq 0.5$ and rapidly increases for $c > 0.9$.}
  \label{fig:confidence}
  \end{figure}
  
\section{Denoising Autoencoder} \label{app:auto}

For the experiment described in Section \ref{sec:common_corruptions}, we trained a denoising autoencoder based on the architecture proposed in \cite{ZhangDenoising17} on the training set of the CIFAR10 dataset. We used $17$ layers and set the number of output channels of the first layer to $64$. We used a batch size of $128$ and an Adam optimizer with a learning rate of $0.001$ \citef{KingmaAdam2014}. We injected Gaussian noise to every image during training with a standard deviation of $0.1$. We used the mean squared error loss and trained the model for $40$ epochs, which was sufficient for convergence.

\section{DRQ Additional Experiments} \label{app:drq_standard_network}

\subsection{Normally Trained Models}

We additionally evaluated DRQ on networks trained without additional data augmentation methods or adversarial training. We observed no considerable differences in robustness or clean accuracy for any of the datasets and models. A more in-depth analysis revealed that the $\epsilon_p$ determined during the calibration step was orders of magnitude smaller than the one obtained for the adversarially trained models. We hypothesize that networks must exhibit at least some robustness to distribution shifts in order for DRQ to affect the robustness.
  
\subsection{Accuracy on Clean Data}

In Table~\ref{tab:clean}, we investigate the difference in accuracy on clean data for different models and datasets for the two DRQ approaches and standard inference. We observe that DRQ has no considerable effect on the clean accuracy of the different models and will explore this behavior in future work.

\begin{table*}
\small
\caption{Accuracy [$\%$] of standard inference and DRQ inference in two different norms. The accuracy is shown for differently trained models on the CIFAR10, and ImageNet datasets. The best accuracy is shown in \textbf{bold} and the second best is \underline{underlined}.}
\centering
\begin{tabular}{lrrrr}
    \toprule
    Models & Training & Standard & DRQ $\ell_{\infty}$ & DRQ $\ell_2$  \\
    \midrule
    \textbf{CIFAR10} \\
    \citef{Wong2020Fast} &  $\ell_{\infty}$-AT & \textbf{83.34} & 82.44 & \underline{83.11} \\
    \citef{Gowal2020Uncovering} & $\ell_{\infty}$-AT & \underline{89.48} & 89.22 & \textbf{89.62} \\
    \citef{Rice2020Overfitting} & $\ell_{2}$-AT & \underline{88.67} & \underline{88.63} & \textbf{88.84} \\
    \citef{hendrycks2021Faces} & No-AT & \underline{95.08} & \underline{95.07} & \textbf{95.28} \\
    \midrule
    \textbf{ImageNet} \\
    \citef{Wong2020Fast} & $\ell_{\infty}$-AT & \textbf{53.8} & \textbf{53.8} & \textbf{53.9} \\
    \citef{hendrycks2021Faces} & No-AT & \textbf{75.8} & \textbf{75.8} & \textbf{75.8} \\
    \bottomrule
    \end{tabular}
    \label{tab:clean}
\end{table*}

\section{Attack Success Rate} \label{app:attack_success}

The success rate of the individual attacks is exemplified in Table \ref{tab:adversarial_wong} for the model proposed in \cite{Wong2020Fast}. Furthermore, we attack the whole DRQ pipeline by including DRQ in the forward and backward pass of an PGD-based adversarial attack. To this end, we substitue the sign function of the gradient in the different steps with the gradient to make the whole pipeline differentiable (this does not affect the performance of DRQ). We further skipped the calibration step in these experiments and used a predefined $\epsilon_p$. For the same amount of attack iterations, this attack achieved a marginally better success rate than our adaptive attack (+0.11\% success rate for the model proposed by \cite{Wong2020Fast}). However, the proposed adaptive attacks were considerably more effective with the same amount of computational complexity. We explored the behavior of the end-to-end attack and observed a vanishing/exploding gradient problem due to the size of the computational graph. This phenomenon was also observed by \citet{Athalye2018Obfuscated}, when attacking models that include an iterative process. We further explored using the score-based black-box square attack on the whole pipeline but this was not more effective.

\begin{table*}[h]
\small
\caption{Accuracy [$\%$] of standard inference and DRQ inference in two different norms against different adversarial attacks for the CIFAR10 model proposed by \cite{Wong2020Fast}. The worst-case accuracy of the combined attacks is shown in the rightmost column.}
\centering
\begin{tabular}{lrrrrrrrrrrr}
    \toprule
Attacks & Clean & FAB & Square & APGD\textsubscript{L2} & APGD\textsubscript{CE} & APGD\textsubscript{Noise} & APGD\textsubscript{Attack} & APGD\textsubscript{DLR} & Worst-case \\
\midrule
Standard & 83.34 & 44.35 & 51.22 & 61.38 & 45.83 & 46.77 & 53.59 & 47.05 & 43.37 \\
DRQ $\ell_2$ & 83.17 & 70.19 & 70.0 & 65.36 & 49.93 & 48.68 & 54.05 & 51.64 & 43.39 \\
DRQ $\ell_{\infty}$ & 81.16 & 71.43 & 72.28 & 72.02 & 58.97 & 52.62 & 53.73 & 62.27 & 48.52 \\
        \bottomrule
    \end{tabular}
    \label{tab:adversarial_wong}
\end{table*}

\section{DRQ Exploration Directions} \label{app:drq_directions}

Figure \ref{fig:drq_directions} shows, the cosine similarity between DRQ direction $\gamma_{\mathrm{DRQ}} = \Hat{x}_{\mathrm{DRQ}} - x$ and the adversarial perturbation direction $\gamma = x_{\mathrm{adv}} - x$ for the model proposed by \citet{Wong2020Fast}, as described in Section \ref{sec:adversarial}. The mean cosine similarity of all attacks is closer or equal to zero. The lowest mean cosine similarity is displayed by the adaptive PGD Noise attack. However, for this attack DRQ was not able to correct the prediction as often (+$5.85\%$ accuracy with DRQ) as for example for the other attacks (e.g., Square: +$21.06\%$ accuracy with DRQ). We argue that the perturbations found by the adaptive PGD Noise attack are more robust. Thus, there exist less perturbation directions that can correct the prediction and DRQ is forced to find the opposing direction to the adversarial perturbation $\gamma$, such that $-\gamma_{\mathrm{DRQ}} \approx \gamma$. This leads to a lower cosine similarity for the PGD Noise attack in \ref{fig:drq_directions}.

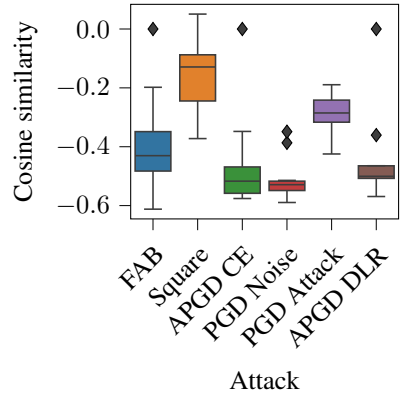
\begin{figure}[h]
  \centering
  \begin{subfigure}{.3\textwidth}
      \centering
\begin{tikzpicture}[
trim axis left, trim axis right
]

\definecolor{color0}{rgb}{0.194607843137255,0.453431372549019,0.632843137254902}
\definecolor{color1}{rgb}{0.881862745098039,0.505392156862745,0.173039215686275}
\definecolor{color2}{rgb}{0.229411764705882,0.570588235294118,0.229411764705882}
\definecolor{color3}{rgb}{0.75343137254902,0.238725490196078,0.241666666666667}
\definecolor{color4}{rgb}{0.578431372549019,0.446078431372549,0.699019607843137}
\definecolor{color5}{rgb}{0.517156862745098,0.358333333333333,0.325980392156863}

\begin{axis}[
tick align=outside,
tick pos=left,
width=\linewidth,
x grid style={white!69.0196078431373!black},
xlabel={Attack},
xmin=-0.5, xmax=5.5,
xtick style={color=black},
xtick={0,1,2,3,4,5},
xticklabel style={rotate=45.0,anchor=north east},
xticklabels={FAB,Square,APGD CE,PGD Noise,PGD Attack,APGD DLR},
y grid style={white!69.0196078431373!black},
ylabel={Cosine similarity},
ymin=-0.64522976949811, ymax=0.0838073171675205,
ytick style={color=black},
ytick={-0.8,-0.6,-0.4,-0.2,0,0.2},
yticklabels={
  \(\displaystyle {\ensuremath{-}0.8}\),
  \(\displaystyle {\ensuremath{-}0.6}\),
  \(\displaystyle {\ensuremath{-}0.4}\),
  \(\displaystyle {\ensuremath{-}0.2}\),
  \(\displaystyle {0.0}\),
  \(\displaystyle {0.2}\)
}
]
\path [draw=white!23.921568627451!black, fill=color0, semithick]
(axis cs:-0.4,-0.482635796070099)
--(axis cs:0.4,-0.482635796070099)
--(axis cs:0.4,-0.34867525100708)
--(axis cs:-0.4,-0.34867525100708)
--(axis cs:-0.4,-0.482635796070099)
--cycle;
\path [draw=white!23.921568627451!black, fill=color1, semithick]
(axis cs:0.6,-0.244527224451303)
--(axis cs:1.4,-0.244527224451303)
--(axis cs:1.4,-0.0876418221741915)
--(axis cs:0.6,-0.0876418221741915)
--(axis cs:0.6,-0.244527224451303)
--cycle;
\path [draw=white!23.921568627451!black, fill=color2, semithick]
(axis cs:1.6,-0.5582055747509)
--(axis cs:2.4,-0.5582055747509)
--(axis cs:2.4,-0.468662843108177)
--(axis cs:1.6,-0.468662843108177)
--(axis cs:1.6,-0.5582055747509)
--cycle;
\path [draw=white!23.921568627451!black, fill=color3, semithick]
(axis cs:2.6,-0.54875060915947)
--(axis cs:3.4,-0.54875060915947)
--(axis cs:3.4,-0.516965791583061)
--(axis cs:2.6,-0.516965791583061)
--(axis cs:2.6,-0.54875060915947)
--cycle;
\path [draw=white!23.921568627451!black, fill=color4, semithick]
(axis cs:3.6,-0.316602736711502)
--(axis cs:4.4,-0.316602736711502)
--(axis cs:4.4,-0.241626307368279)
--(axis cs:3.6,-0.241626307368279)
--(axis cs:3.6,-0.316602736711502)
--cycle;
\path [draw=white!23.921568627451!black, fill=color5, semithick]
(axis cs:4.6,-0.507466807961464)
--(axis cs:5.4,-0.507466807961464)
--(axis cs:5.4,-0.464789167046547)
--(axis cs:4.6,-0.464789167046547)
--(axis cs:4.6,-0.507466807961464)
--cycle;






\addplot [semithick, white!23.921568627451!black, forget plot]
table {%
0 -0.482635796070099
0 -0.612091720104218
};
\addplot [semithick, white!23.921568627451!black, forget plot]
table {%
0 -0.34867525100708
0 -0.197981804609299
};
\addplot [semithick, white!23.921568627451!black, forget plot]
table {%
-0.2 -0.612091720104218
0.2 -0.612091720104218
};
\addplot [semithick, white!23.921568627451!black, forget plot]
table {%
-0.2 -0.197981804609299
0.2 -0.197981804609299
};
\addplot [black, mark=diamond*, mark size=2.5, mark options={solid,fill=white!23.921568627451!black}, only marks, forget plot]
table {%
0 0
};
\addplot [semithick, white!23.921568627451!black, forget plot]
table {%
1 -0.244527224451303
1 -0.372322499752045
};
\addplot [semithick, white!23.921568627451!black, forget plot]
table {%
1 -0.0876418221741915
1 0.0506692677736282
};
\addplot [semithick, white!23.921568627451!black, forget plot]
table {%
0.8 -0.372322499752045
1.2 -0.372322499752045
};
\addplot [semithick, white!23.921568627451!black, forget plot]
table {%
0.8 0.0506692677736282
1.2 0.0506692677736282
};
\addplot [semithick, white!23.921568627451!black, forget plot]
table {%
2 -0.5582055747509
2 -0.575908184051514
};
\addplot [semithick, white!23.921568627451!black, forget plot]
table {%
2 -0.468662843108177
2 -0.347975522279739
};
\addplot [semithick, white!23.921568627451!black, forget plot]
table {%
1.8 -0.575908184051514
2.2 -0.575908184051514
};
\addplot [semithick, white!23.921568627451!black, forget plot]
table {%
1.8 -0.347975522279739
2.2 -0.347975522279739
};
\addplot [black, mark=diamond*, mark size=2.5, mark options={solid,fill=white!23.921568627451!black}, only marks, forget plot]
table {%
2 0
};
\addplot [semithick, white!23.921568627451!black, forget plot]
table {%
3 -0.54875060915947
3 -0.589447319507599
};
\addplot [semithick, white!23.921568627451!black, forget plot]
table {%
3 -0.516965791583061
3 -0.514293909072876
};
\addplot [semithick, white!23.921568627451!black, forget plot]
table {%
2.8 -0.589447319507599
3.2 -0.589447319507599
};
\addplot [semithick, white!23.921568627451!black, forget plot]
table {%
2.8 -0.514293909072876
3.2 -0.514293909072876
};
\addplot [black, mark=diamond*, mark size=2.5, mark options={solid,fill=white!23.921568627451!black}, only marks, forget plot]
table {%
3 -0.386733680963516
3 -0.348839998245239
};
\addplot [semithick, white!23.921568627451!black, forget plot]
table {%
4 -0.316602736711502
4 -0.424588739871979
};
\addplot [semithick, white!23.921568627451!black, forget plot]
table {%
4 -0.241626307368279
4 -0.189326524734497
};
\addplot [semithick, white!23.921568627451!black, forget plot]
table {%
3.8 -0.424588739871979
4.2 -0.424588739871979
};
\addplot [semithick, white!23.921568627451!black, forget plot]
table {%
3.8 -0.189326524734497
4.2 -0.189326524734497
};
\addplot [semithick, white!23.921568627451!black, forget plot]
table {%
5 -0.507466807961464
5 -0.569108307361603
};
\addplot [semithick, white!23.921568627451!black, forget plot]
table {%
5 -0.464789167046547
5 -0.464593052864075
};
\addplot [semithick, white!23.921568627451!black, forget plot]
table {%
4.8 -0.569108307361603
5.2 -0.569108307361603
};
\addplot [semithick, white!23.921568627451!black, forget plot]
table {%
4.8 -0.464593052864075
5.2 -0.464593052864075
};
\addplot [black, mark=diamond*, mark size=2.5, mark options={solid,fill=white!23.921568627451!black}, only marks, forget plot]
table {%
5 -0.360520392656326
5 0
};
\addplot [semithick, white!23.921568627451!black, forget plot]
table {%
-0.4 -0.430354356765747
0.4 -0.430354356765747
};
\addplot [semithick, white!23.921568627451!black, forget plot]
table {%
0.6 -0.12899025529623
1.4 -0.12899025529623
};
\addplot [semithick, white!23.921568627451!black, forget plot]
table {%
1.6 -0.516935884952545
2.4 -0.516935884952545
};
\addplot [semithick, white!23.921568627451!black, forget plot]
table {%
2.6 -0.528887212276459
3.4 -0.528887212276459
};
\addplot [semithick, white!23.921568627451!black, forget plot]
table {%
3.6 -0.285254746675491
4.4 -0.285254746675491
};
\addplot [semithick, white!23.921568627451!black, forget plot]
table {%
4.6 -0.500413045287132
5.4 -0.500413045287132
};
\end{axis}

\end{tikzpicture}
  \end{subfigure}
  \caption{Cosine similarity between the DRQ direction $\gamma_{\mathrm{DRQ}} = \Hat{x}_{\mathrm{DRQ}} - x$ and the adversarial perturbation direction $\gamma = x_{\mathrm{adv}} - x$ for the CIFAR10 model proposed in \cite{Wong2020Fast}. Negative values indicate that DRQ moves the adversarial example closer to the original input.}
  \label{fig:drq_directions}
  \end{figure}




\section{Exploration and Quantification Step} \label{app:ablation}

We conducted additional experiments to assess if the adversarial optimization in the exploration and quantification steps of DRQ are necessary to increase robustness. To this end, we replaced the exploration step with random noise augmentations, which we describe as \textit{random exploration}. We also replaced the quantification step with random noise augmentations. Instead of searching for the sample that shows the lowest robustness, we inject $20$ different random noise augmentations to every input and select the augmented sample with the lowest robustness. We refer to this procedure as \textit{random quantification}. Table \ref{tab:ablation_noise} summarizes the results of replacing either the exploration or the quantification step or both with their random versions. The experiment was conducted for the model proposed in \cite{Wong2020Fast} on the CIFAR10-C dataset and against the adversarial attack ensemble on CIFAR10. Replacing either the exploration or quantification step reduces the robustness in all cases. This highlights the importance of accurately estimating the robustness of the local decision regions in both steps of the DRQ algorithm. 

\begin{table*}[h]
\small
\caption{Ablation study on the accuracy differences [$\%$] between DRQ and standard prediction for different region exploration and region quantification approaches for the model proposed in \cite{Wong2020Fast}. Here, either the exploration or quantification step or both are replaced by a random approximation.}
\centering
\begin{tabular}{lrr|rr}
\toprule
Data & \multicolumn{2}{c}{CIFAR10-C} & \multicolumn{2}{c}{CIFAR10 worst-case} \\
\midrule
 & Exploration - Random & Exploration - $\ell_{\infty}$  & Exploration - Random & Exploration-$\ell_{\infty}$ \\
Quantification - Random & -0.1 & -0.5 & -0.75 & -0.27 \\
Quantification - $\ell_{\infty}$ & -0.25 & \textbf{+2.03} & -0.35 & \textbf{+5.85} \\
    \end{tabular}
    \label{tab:ablation_noise}
\end{table*}

\end{document}